\documentclass[journal]{IEEEtran}

\usepackage{flushend}
\usepackage{cite}
\usepackage{caption}
\usepackage{subcaption}
\def\BibTeX{{\rm B\kern-.05em{\sc i\kern-.025em b}\kern-.08em
    T\kern-.1667em\lower.7ex\hbox{E}\kern-.125emX}}
    
\usepackage{amsmath,amssymb,amsfonts,bm,amsthm}
\usepackage{algorithmic}
\usepackage[ruled,vlined]{algorithm2e}
\usepackage{siunitx}
\newtheorem{theorem}{Theorem}[section]
\newtheorem{lemma}[theorem]{Lemma}

\usepackage{pgfplots}
\usepackage{grffile}
\pgfplotsset{compat=newest}
\usetikzlibrary{math,calc,plotmarks,arrows.meta}
\usetikzlibrary{arrows.meta}
\usepgfplotslibrary{patchplots}
\usepackage{graphicx}
\usepackage{xcolor}
\usepackage{float}
\usepackage[hidelinks]{hyperref}
\usepackage{cleveref}
\usepackage{colortbl}

\DeclareMathOperator*{\argmax}{arg\,max}

\DeclareMathOperator*{\R}{\mathbb{R}}
\DeclareMathOperator*{\E}{\mathbb{E}}
\DeclareMathOperator*{\C}{\mathcal{C}}

\renewcommand{\Vec}{\mathbf}

\DeclareMathOperator{\w}{\Vec{w}}  
\DeclareMathOperator{\mpos}{\bm{\mu}_+}
\DeclareMathOperator{\mneg}{\bm{\mu}_-}
\DeclareMathOperator{\Dmu}{\mpos - \mneg}
\DeclareMathOperator{\Sb}{\bar{\Sigma}}
\DeclareMathOperator{\Sh}{\hat{\Sigma}}
\DeclareMathOperator{\mb}{\bar{\bm{\mu}}}
\DeclareMathOperator{\mD}{\bm{\mu}_\Delta}
\DeclareMathOperator{\diag}{diag}
\newcommand{\notprop}{\propto\kern-1\@ptsize pt \diagup}

\SetKwInput{KwInput}{Input}                
\SetKwInput{KwOutput}{Output}              

\begin{document}
\title{Minimally Informed Linear Discriminant Analysis:\newline training an LDA model with unlabelled data}

\author{Nicolas~Heintz,
        Tom~Francart,
        and~Alexander~Bertrand,~\IEEEmembership{Senior~Member,~IEEE}%
\thanks{N. Heintz and A. Bertrand are with KU Leuven, Department of Electrical Engineering (ESAT), STADIUS Center for Dynamical Systems, Signal Processing and Data Analytics, Belgium.}%
\thanks{N. Heintz and T. Francart are with KU Leuven, Department of Neurosciences, ExpORL, Belgium.}%
\thanks{This research is funded by Aspirant Grant 1S31522N (for N. Heintz) from the Research Foundation - Flanders (FWO), the Research Foundation - Flanders (FWO) project No G0A4918N and G081722N, the European Research Council (ERC) under the European Union’s Horizon 2020 research and innovation program (grant agreement No 802895 and grant agreement No 637424), and the Flemish Government (AI Research Program).
The scientific responsibility is assumed by its authors.\\
N. Heintz, T. Francart and A. Bertrand are also affiliated with Leuven.AI - KU Leuven institute for AI, Belgium.}}%
\maketitle
\begin{abstract}
Linear Discriminant Analysis (LDA) is one of the oldest and most popular linear methods for supervised classification problems. In this paper, we demonstrate that it is possible to compute the exact projection vector from LDA models based on unlabelled data, if some minimal prior information is available. More precisely, we show that only one of the following three pieces of information is actually sufficient to compute the LDA projection vector if only unlabelled data are available: (1) the class average of one of the two classes, (2) the difference between both class averages (up to a scaling), or (3) the class covariance matrices (up to a scaling). These theoretical results are validated in numerical experiments, demonstrating that this minimally informed Linear Discriminant Analysis (MILDA) model closely matches the performance of a supervised LDA model. Furthermore, we show that the MILDA projection vector can be computed in a closed form with a computational cost comparable to LDA and is able to quickly adapt to non-stationary data, making it well-suited to use as an adaptive classifier.
\end{abstract}
\section{Introduction}
\label{sec:intro}
Fisher's Linear Discriminant Analysis (LDA) model is a standard algorithm in data science for linear classification \cite{Hastie2009}. Its goal is to find the optimal linear projection of the data such that two classes are maximally separated. Despite the emergence of more advanced techniques, LDA remains a very popular data analysis tool in all scientific fields \cite{Bouveyron2019, Auguin2021, Liu2020, Antuvan2019, Ricciardi2020, Zhu2022, Elkhalil2020} because it is easy to understand, cheap to compute, and has an elegant and interpretable formulation. Its low computational complexity also makes it a popular tool for signal processing, in particular in real-time or low-resource environments \cite{Liu2020,Antuvan2019}. 

In his original LDA paper, Fisher proposed a transformation that maximises the distance between class averages and minimises the variance within each class \cite{Fisher1936}. Computing Fisher's LDA model requires four (or two) parameters: the two class averages (or their difference) and the class covariance matrices of both classes (or their sum). These parameters can either be known via prior knowledge or estimated using labelled data. Herein lies one of the main drawbacks of LDA (and any other supervised model): as soon as one of these parameters is unknown, a labelled dataset is required to train the LDA model. For many applications, this requires manual labelling by experts, which is time-consuming and costly. Furthermore, when a problem is non-stationary, i.e. when the statistics of the data change over time, it is often important to regularly re-train the model, which only exacerbates the aforementioned issues.

Alternatively, one can resort to unsupervised classification algorithms that attempt to classify given data without requiring labels. Such unsupervised classification problems are generally tackled with clustering algorithms such as, e.g., Gaussian Mixture Models (GMMs) \cite{Duda1974}, K-means \cite{Lloyd1982}, and its extension K-means++ \cite{Arthur2007}. However, the modelling power and robustness of LDA, in particular in high-dimensional spaces, have motivated the incorporation of LDA in such unsupervised frameworks. A popular approach is to iteratively estimate a set of pseudo-labels using some unsupervised clustering method and then train the LDA model with these labels \cite{DeLaTorre2006, Deng2019, Wang2021}. 

Another solution to the labelling problem is to only label a limited amount of samples and then train the model on a small number of labelled samples and a large number of unlabelled samples, which is called semi-supervised learning. For semi-supervised LDA, the small number of labelled samples can then give a rough indication of where the two classes are supposed to be located. This information can then be merged with a much better description of the general geometry of the data delivered by the vast amount of unlabelled data through some objective function to provide an estimation of the LDA projection, as done in \cite{Cai2007, Wang2016}. 

All these proposed methods require (pseudo-)labels in order to compute Fisher's LDA projections. However, there exist many scenarios where some of the relevant ground-truth statistics of the data are known (or can be estimated without labels), but not all of them:
\begin{itemize}
    \item Only the \textbf{class average} of a \textbf{single class} is known: This could occur, e.g., when there is a (partial) model for one of the classes, but not for the other. One example could be a detection problem where the noise is known to be zero-mean, but where its covariance structure is unknown. In the absence of a detection target (class 0), the mean vector is then known to be 0. Another example is hypothesis testing, where the mean of the null hypothesis distribution is assumed to be known.
    \item Only the \textbf{relative difference} between the \textbf{class averages}  is known (up to a scaling): This is relevant when it is known in what direction the features are expected to move when shifting from one class to the other, without knowing how large the change is expected to be, what the exact class averages are or what variance is expected. This is the case, for example, when a binary amplitude-modulated signal must be detected with a sensor array. The steering vector of this signal is often known up to an unknown scalar, while there can simultaneously be ambiguity in the amplitude of the source signal.
    \item Only the \textbf{covariances} for each of the \textbf{two classes} are known (up to a scaling): Such a scenario is common when, e.g., the noise covariance statistics are known (possibly up to an unknown calibration factor). This also includes the common case where the noise is independent and identically distributed across the feature space. 
\end{itemize}

In general, the availability of one such ground-truth statistic is not enough to train a classifier from unlabelled data, in which case the remaining statistics still need to be estimated from labelled data. One exception is the last case, with the additional assumption that the classes have a Gaussian distribution and that they are sufficiently separated. In this case, it is enough to know the class covariance matrices to compute LDA without labels \cite{Pena2000, Martin-Clemente2020}.

In this paper, we present a mathematical formulation that actually makes it possible to compute the LDA projection with only \textit{one} of the aforementioned ground-truth statistics, which we refer to as `minimally informed' LDA (MILDA). In contrast to \cite{Pena2000, Martin-Clemente2020}, it does not assume Gaussian class distributions, does not require a good class separability and is applicable if any of the aforementioned ground-truth statistics are known. 

The proposed MILDA framework still results in closed-form expressions, which are elegant and cheap to compute. This is in line with the main advantages of the standard LDA. MILDA optimises a slightly more general objective function than LDA, and has the original LDA objective as a special case when either the classes are balanced or if the homoscedasticity assumption is satisfied. However, even if neither of these assumptions hold, MILDA empirically still achieves similar accuracies as LDA. Furthermore, MILDA becomes more robust to estimation errors in the ground-truth statistic when the classes are less separated or imbalanced. This is in stark contrast with most unsupervised clustering methods. Up to our knowledge, such a result has not been described.

The aforementioned properties make MILDA especially suited for use in adaptive modelling. Typical examples are, e.g., classifying a non-stationary signal with stationary noise or vice versa. The stationary statistics can then be estimated and used as ground-truth statistics, while MILDA automatically adapts to the non-stationary statistics. 

The outline of the paper is as follows. In Section \ref{sec:MILDA}, we describe the main mathematical framework behind the MILDA model. In Section \ref{sec:sensitivity} we provide a theoretical analysis of the sensitivity of MILDA to errors in the used ground-truth statistic. In Section \ref{sec:experiments}, we evaluate MILDA in numerical experiments and demonstrate its performance on a target signal detection task and on non-stationary data. The computational cost of MILDA is briefly discussed in Section \ref{sec:CompCost}. Finally, in Section \ref{sec:conclusion}, we summarise the main conclusions from this paper and discuss possibilities for future work. 

As a notational convention, we use regular letters ($a$) to denote scalars, bold letters ($\Vec{a}$) for vectors, capital letters ($A$) for constants and matrices (depending on the context) and Greek letters ($\alpha$) for parameters and latent scalars to shorten equations.

\section{Minimally informed LDA (MILDA)}
\label{sec:MILDA}
In this section, we will first briefly review the main mathematical framework of LDA as proposed by Fisher \cite{Fisher1936}. In Section \ref{subsec:LDAwoLabels}, we show how the same projection can be found with the temporary (and overly strict) assumption that the two class averages are proportionate to each other. This generates an intermediate result that will be needed in the derivation of MILDA. Since this assumption of proportionality is unrealistic in real-life applications, we show in Section \ref{subsec:Transformations} how this assumption can be dropped when only a single ground-truth statistic is known. No other assumptions on the data are necessary to prove that the MILDA and LDA projections are equivalent.

\subsection{Linear Discriminant Analysis}
\label{subsec:LDA}
Let $\{\Vec{x}_i\}_i$ with $\Vec{x}_i \in \R^D$ be a set of $D$-variate samples, where each sample belongs to one of two classes. The goal of Linear Discriminant Analysis (LDA) is to find a projection vector $\w \in \R^D$ that maximally separates two classes after projecting each data point onto $\w$. An intuitive way to do this is by searching for the projection that maximises the distance between the class averages and minimises the variance within each class. By writing these two joint objectives as a ratio, Fisher proposed the following objective:

\begin{equation}
\label{eq:objectiveLDA}
    \argmax_{\w} \frac{\w^\top(\Dmu)}{\w^\top (\Sigma_+ + \Sigma_-) \w},
\end{equation}
where $\mpos$, $\mneg$ are the class averages and $\Sigma_+$, $\Sigma_-$ are the covariance matrices for each class \cite{Fisher1936}. If a labelled set of samples $\Vec{x}_i$ is available, then the class averages and class covariances can be estimated as\footnote{We use the biased covariance estimator throughout this paper for notational convenience.}:
\begin{align}
\bm{\mu}_\pm &\triangleq \frac{1}{N_\pm} \sum_{i\in\C_\pm} \Vec{x}_i \label{eq:mupm}\\
\Sigma_\pm    &\triangleq \frac{1}{N_\pm} \sum_{i\in\C_\pm} (\Vec{x}_i-\bm{\mu}_\pm)(\Vec{x}_i-\bm{\mu}_\pm)^\top, \label{eq:SposOriginal}
\end{align}
where $\C_+$ and $\C_-$ contain the sample indices of the positive and negative class, respectively, with cardinalities $N_+=|\C_+|$ and $N_-=|\C_-|$. The solution of \eqref{eq:objectiveLDA} is then given by:
\begin{equation}
\label{eq:LDA}
    \w^* \propto (\Sigma_+ + \Sigma_-)^{-1}(\Dmu),
\end{equation}
assuming $\Sigma_+ + \Sigma_-$ is full rank\footnote{If the summed covariance matrix is not full rank, a dimensionality reduction should be performed first, e.g., via principal component analysis.} \cite{Fisher1936}. Here we use the notation $\propto$ to denote that the left- and right-hand sides are equal up to an arbitrary non-zero scaling.

It was later shown that this solution minimises the classification error if both classes have a Gaussian distribution and if they are homoscedastic ($\Sigma_+ = \Sigma_-$), i.e., no other classifier can result in better classification performance \cite{Hastie2009}. In this case, $\w \propto \Sigma^{-1}(\Dmu)$, with $\Sigma = \Sigma_+ = \Sigma_-$. Confusingly, this solution is often also referred to as \textit{Linear Discriminant Analysis}, but it is an edge case of the more general LDA formulation provided by Fisher. We note that Gaussianity and homoscedasticity should not be viewed as limiting assumptions in order to use LDA, i.e., without these assumptions, the LDA objective \eqref{eq:objectiveLDA} remains relevant and often results in a good classifier \cite{Bouveyron2019, Hastie2009}.

In this paper, we will show that MILDA provides a projection vector that optimises a slightly more general objective (compared to \eqref{eq:objectiveLDA}): 
\begin{equation}
\label{eq:objectiveMILDA}
    \argmax_{\w} \frac{\w^\top(\Dmu)}{\w^\top (q\Sigma_+ + (1-q)\Sigma_-) \w},
\end{equation}
with $q = N_+/N$ the fraction of samples in a dataset that belong to the positive class and $N=N_++N_-$ the total number of samples in a class. The solution of \eqref{eq:objectiveMILDA} is given by: 
\begin{align}
    \w^* &\propto \Sh^{-1}(\Dmu), \label{eq:LDA_MILDAObjective}\\  
    \Sh &\triangleq q\Sigma_+ + (1-q)\Sigma_-. \label{eq:ShDefinition}
\end{align}
Note that \eqref{eq:LDA_MILDAObjective} is identical to the original LDA solution \eqref{eq:LDA} if either $\Sigma_+ \propto \Sigma_-$ (e.g. when the classes are homoscedastic) or if $q=0.5$, meaning that the classes are balanced. When this is not the case, MILDA will put more weight on minimising the class variance of the class containing more samples compared to the class containing fewer samples than LDA. This could theoretically be beneficial when little data are available, as in that case, the covariance matrices are sensitive to estimation errors. However, in practice we found no noteworthy difference in performance between \eqref{eq:LDA_MILDAObjective} and \eqref{eq:LDA}, as demonstrated in Section \ref{sec:experiments}.

\subsection{Preliminaries to MILDA}
\label{subsec:LDAwoLabels}
The projection from \eqref{eq:LDA_MILDAObjective} requires knowledge of $\Sigma_+$, $\Sigma_-$, the class difference $\Dmu$, and $q$. In practice, these parameters can usually only be estimated on labelled data in a supervised setting. However, in this section we will show that the solution \eqref{eq:LDA_MILDAObjective} can be computed without labels if the two class averages are proportional\footnote{We use $\Vec{x}\sim \Vec{y}$ to denote that either $\Vec{x}=\alpha \Vec{y}$ or $\Vec{y}=\beta \Vec{x}$ with $\alpha, \beta\in \R$, whereas we use $\Vec{x}\propto \Vec{y}$ to denote that $\Vec{x}=\alpha \Vec{y}$ with $\alpha\in\R\backslash\{0\}$. In the former case, it is allowed that either $\Vec{x}$ or $\Vec{y}$ is the zero vector, whereas in the latter both vectors should be non-zero.}, i.e., $\mpos \sim \mneg$. While the latter assumption may seem artificial and overly strict at first sight, we will show in Section \ref{subsec:Transformations} that it can be relaxed to more realistic conditions based on knowledge of some ground-truth statistics.  

\begin{lemma}
\label{lemma:MILDA_pure}
If $\mpos \sim \mneg$, then the solution of \eqref{eq:objectiveMILDA} can be computed as
\begin{align}
    &\w^* \propto \Sb^{-1}\mb \label{eq:MILDA_pure}\\
    &\text{with} \mb = \frac{1}{N}\sum_{i=1}^N \Vec{x}_i\label{eq:mb}\\
    &\phantom{\text{with}} \Sb = \frac{1}{N}\sum_{i=1}^N(\Vec{x}_i-\mb)(\Vec{x}_i-\mb)^\top \label{eq:SbOriginal}
\end{align}
\end{lemma}

\begin{proof} 
\label{subsubsec:proofMILDA}
With some algebraic manipulations, we can rewrite the covariance matrix $\Sb$ in \eqref{eq:SbOriginal} and class covariance matrices $\Sigma_+$ and $\Sigma_-$ in \eqref{eq:SposOriginal} as:

\begin{align}
    \Sb         &= \frac{1}{N}      \sum_{i=1}^N \Vec{x}_i\Vec{x}_i^\top -\mb\mb^\top \label{eq:Sb}\\
    \Sigma_+    &= \frac{1}{qN}     \sum_{i\in\C_+} \Vec{x}_i\Vec{x}_i^\top -\mpos\mpos^\top \label{eq:Spos}\\
    \Sigma_-    &= \frac{1}{(1-q)N} \sum_{i\in\C_-} \Vec{x}_i\Vec{x}_i^\top -\mneg\mneg^\top \label{eq:Smin}.
\end{align}
Based on \eqref{eq:ShDefinition}, \eqref{eq:Sb}-\eqref{eq:Smin}, it can be verified that: 
\begin{align}
    \Sb &= q(\Sigma_+ + \mpos\mpos^\top) + (1-q)(\Sigma_- + \mneg\mneg^\top)-\mb\mb^\top \nonumber\\
        &= \Sh + q\mpos\mpos^\top + (1-q)\mneg\mneg^\top-\mb\mb^\top. \label{eq:SbIFOSpm}
\end{align}
Furthermore, from \eqref{eq:mupm} and \eqref{eq:mb}, we can also write

\begin{equation}
        \mb = q\mpos + (1-q)\mneg. \label{eq:mbIFOmpm}\\
\end{equation}
Since $\mpos \sim \mneg$, we define\footnote{If $\mpos=\Vec{0}$, the roles of $\mpos$ and $\mneg$ should be switched in the proof.} $\mneg = \alpha\mpos$ with $\alpha \in \R\backslash\{1\}$. Note that $\alpha=1$ is excluded, since this would imply that the mean of both classes coincides, in which case \eqref{eq:objectiveMILDA} is ill-defined. By substituting $\mneg = \alpha\mpos$ in \eqref{eq:SbIFOSpm} and \eqref{eq:mbIFOmpm}, we can simplify $\mb$ and $\Sb$ as: 

\begin{alignat*}{2}
    \mb &= (q+(1-q)\alpha)\mpos &&\triangleq \beta\mpos\\
    \Sb &= \Sh + (q + (1-q)\alpha^2 - \beta^2) \mpos\mpos^\top &&\triangleq \Sh + \gamma\mpos\mpos^\top,
\end{alignat*}
with $\gamma=q(1-q)(\alpha-1)^2>0$.

When plugging this in the right-hand side of \eqref{eq:MILDA_pure}, we get.
\begin{align}
    \w^*  &\propto \Sb^{-1}\mb \nonumber\\
        &\propto (\Sh + \gamma\mpos\mpos^\top)^{-1}\mpos, \label{eq:wBeforeWoodbury}
\end{align}
where $\beta$ is absorbed in the $\propto$ sign.

We will now show that \eqref{eq:wBeforeWoodbury} is indeed a solution of \eqref{eq:objectiveMILDA}. First, the matrix inverse in this expression can be rewritten using the Sherman-Morrison identity, which is a special case of the Woodbury identity for rank-1 updates \cite{Woodbury1950}. This allows us to write \eqref{eq:wBeforeWoodbury} as: 
\begin{align}
    \w^* &\propto \left(\Sh^{-1} - \frac{\gamma\Sh^{-1}\mpos\mpos^\top\Sh^{-1}}{1+\gamma\mpos^\top\Sh^{-1}\mpos}\right)\mpos \nonumber\\
       &\propto \Sh^{-1}\mpos -\frac{\gamma\mpos^\top\Sh^{-1}\mpos}{1+\gamma\mpos^\top\Sh^{-1}\mpos}\Sh^{-1}\mpos \nonumber\\
       &\propto \Sh^{-1}\mpos \label{eq:w=Sh^-1mpos},
\end{align}
where the irrelevant scaling factor between the left- and right-hand side were absorbed in the $\propto$ sign.

Finally, since $\mpos - \mneg = (1-\alpha)\mpos \propto \mpos$, we can rewrite \eqref{eq:w=Sh^-1mpos} as:
\begin{equation*}
    \w^* \propto \Sh^{-1}(\Dmu).
\end{equation*}
This results in \eqref{eq:LDA_MILDAObjective}, which is a solution of \eqref{eq:objectiveMILDA}, proving the lemma. 
\end{proof}
As stated before, the vector $\w^*$ in \eqref{eq:LDA_MILDAObjective} or \eqref{eq:MILDA_pure} is proportional to Fisher's LDA defined in \eqref{eq:LDA} if the classes are balanced ($q=0.5$) or if $\Sigma_+ \propto \Sigma_-$. Otherwise, this vector optimises the objective formulated in \eqref{eq:objectiveMILDA}.

\subsection{The MILDA framework}
\label{subsec:Transformations}
Section \ref{subsec:LDAwoLabels} showed that the LDA projection vector can be computed using the label-free expression in \eqref{eq:MILDA_pure} if the two class averages are proportional to each other or if either class is zero-mean, i.e. if $\mpos \sim \mneg$. However, in practice this rather artificial assumption is rarely satisfied. Fortunately, it is possible to satisfy this assumption by transforming the data using one out of the three possible ground-truth statistics: the class average of a single class, the relative difference between the class averages (up to a scaling) or the covariances for each of the two classes (up to a scaling). The latter case also assumes either balanced classes (or knowledge of the class imbalance), or a relaxed homoscedasticity condition (i.e. $\Sigma_+ \propto \Sigma_-$). We will go over these cases one by one and derive the required transformations. We will denote the transformed data and its corresponding statistics using the notation $(.)'$, i.e., the mean of the positive class after transformation is denoted as $\bm{\mu}_+'$. All transformations are summarised in Table \ref{tab:MILDAPreProc} for convenience of the reader. The MILDA algorithm,  which is summarised in Algorithm \ref{alg:MILDA}, then simply requires applying the corresponding transformation on the data and then computing \eqref{eq:MILDA_pure}.

\subsubsection{The class average of a single class is known}
Assume without loss of generality that the class average $\mpos$ is known. In this case, you can guarantee that $\mpos \sim \mneg$ after transforming each feature vector $\Vec{x}_i$ to:
\begin{equation}
    \label{eq:transClassKnown}
    \Vec{x}'_i \longleftarrow \Vec{x}_i - \mpos.
\end{equation}
Indeed, in this case $\bm{\mu}'_+ =\Vec{0}$. Therefore $\bm{\mu}'_+ = 0\bm{\mu}'_-$, making $\bm{\mu}'_+ \sim \bm{\mu}'_-$ trivially true. The new global average $\mb'=\mb - \mpos \neq \Vec{0}$ as long as $\mneg \neq \mpos$, which ensures that the MILDA transformation from \eqref{eq:MILDA_pure} has a solution. If $\mneg = \mpos$, the two classes are not separable with either LDA or MILDA and another classifier should be used. 

\subsubsection{The relative difference between the class averages is known}
In this case, $\mpos-\mneg = \delta\Vec{d}$, with $\delta \in \R$ an unknown scalar and $\Vec{d}$ a known vector. It is thus unknown by how much the classes are separated. We can guarantee that $\mpos \sim \mneg$ by transforming each data sample $\Vec{x}_i$ to:

\begin{equation}
\label{eq:RelDifTransformation}
    \Vec{x}'_i \longleftarrow \Vec{x}_i - (\mb - \alpha\Vec{d}), \alpha \neq 0.
\end{equation}
Though parameter $\alpha\neq0$ can be freely chosen, it is best to choose $\alpha$ such that $\|\alpha \Vec{d}\| \approx \|\Vec{x}_i\|$ to avoid ill-conditioning. 

After the transformation and using \eqref{eq:mbIFOmpm}, we can write the new class averages as: 
\begin{align*}
    \bm{\mu}'_+  &= \mpos - \mb + \alpha\Vec{d} = (1-q)(\mpos - \mneg) + \alpha\Vec{d}\\ 
            &= ((1-q)\delta+\alpha)\Vec{d}\\
    \bm{\mu}'_-  &= \mneg - \mb + \alpha\Vec{d} = -q\mpos + q\mneg+\alpha\Vec{d}\\
            &=(-q\delta+\alpha)\Vec{d}.
\end{align*}
 This shows indeed that $\bm{\mu}_+'\sim\bm{\mu}_-'$. Furthermore, using \eqref{eq:mbIFOmpm} for the transformed data, we see that $\mb' \neq \Vec{0}$ as long as $\alpha \neq 0$.
 
\subsubsection{The covariances for each of the two classes are known}
\label{subsubsec:directionClasses}
In this section, we assume that the covariances $\Sigma_+$ and $\Sigma_-$ are known (up to a scaling), i.e., we assume the availability of the matrices $S_+ =\beta \Sigma_+$ and $S_- = \beta  \Sigma_-$  with $\beta$ an unknown scalar. In addition, we assume that the classes are balanced or that the class proportions (i.e. $q$) are known in case they are unbalanced. We will later see that this latter assumption can be removed if a relaxed homoscedasticity condition (i.e. $\Sigma_+ \propto \Sigma_-$) is satisfied.

We propose the following double transformation, which will later be shown (see Lemma \ref{lemma:covTrans}) to result in $\bm{\mu}_+''\sim\bm{\mu}_-''$:

\begin{align}
    &\Vec{x}_i' \longleftarrow \hat{S}^{-1/2}\Vec{x}_i \label{eq:covCorrectionFirstT}\\
    &\Vec{x}_i'' \longleftarrow \Vec{x}_i' - (\mb'-\alpha\Vec{d}), \alpha \neq 0, \label{eq:covCorrection}\\
    &\text{with}\quad \hat{S} = qS_+ + (1-q)S_-\label{eq:Shoedje}\propto\Sh\\
    &\phantom{\text{with}}\quad \mb' = \frac{1}{N}\sum_{i=1}^N\Vec{x}_i'\nonumber\\ 
    &\phantom{\text{with}}\quad \Sb' = \frac{1}{N}\sum_{i=1}^N\Vec{x}_i'\Vec{x}_i'^\top-\mb'\mb'^\top\nonumber\\
    &\phantom{\text{with}}\quad \Sb'\Vec{d} = \lambda_{max}\Vec{d}, \nonumber
\end{align}
where $\lambda_{max}$ is the largest eigenvalue of $\Sb'$. In other words, we first transform the data via \eqref{eq:covCorrectionFirstT}. Then, we compute the eigenvector $\Vec{d}$ that corresponds to the largest eigenvalue $\lambda_{max}$ from the new covariance matrix $\Sb'$. In Lemma \ref{lemma:covTrans} below, we show that this vector $\Vec{d}$ has the same direction as $\bm{\mu}'_+-\bm{\mu}'_-$. Therefore, once this direction is known, we can apply the same transformation as \eqref{eq:RelDifTransformation} to the data $\Vec{x}_i$, resulting in \eqref{eq:covCorrection}. 

\begin{lemma}
\label{lemma:covTrans}
Let $\Sb$ be defined as in \eqref{eq:Sb}, and let $\hat{S}$ be defined as in \eqref{eq:Shoedje}. If $\Vec{d}$ is the eigenvector of $\Sb'=\hat{S}^{-1/2}\Sb\hat{S}^{-1/2}$ with the largest eigenvalue, then:
\begin{equation*}
    \Vec{d} \propto \hat{S}^{-1/2}(\mpos-\mneg) = \bm{\mu}'_+-\bm{\mu}'_-.
\end{equation*}
\end{lemma}

\begin{proof}
First, we will compute the effect of the first transformation on the covariance matrix. Substituting \eqref{eq:mbIFOmpm} in \eqref{eq:SbIFOSpm}, we obtain

\begin{align}
    \Sb &= \Sh + q(1-q)\left(\mpos\mpos^\top +\mneg\mneg^\top - 2\mpos\mneg^\top\right)\nonumber \\
        &= \Sh + q(1-q)(\mpos-\mneg)(\mpos-\mneg)^\top.
\label{eq:Sigma1Rank1}
\end{align}
After a transformation $\Vec{x}'_i \longleftarrow A\Vec{x}_i$, the new average and covariance are respectively $\mb' = A\mb$ and $\Sb' = A \Sb A^\top$. In this case, \eqref{eq:Sigma1Rank1} becomes: 

\begin{align}
    \Sb'&= A \Sh A^\top + q(1-q)(A\mpos-A\mneg)(A\mpos-A\mneg)^\top.\nonumber\\
        &\triangleq A\Sh A^\top + q(1-q)(\bm{\mu}'_+-\bm{\mu}'_-)(\bm{\mu}'_+-\bm{\mu}'_-)^\top. \label{eq:thirdTransSb'}
\end{align}
Since $\hat{S}$ is the sum of two symmetric positive definite matrices, it is also symmetric positive definite. Therefore, a symmetric matrix $\hat{S}^{-1/2}$ exists. From \eqref{eq:Shoedje}, we know that $\hat{S}^{-1/2}\propto\Sh^{-1/2}$. By defining $A=\hat{S}^{-1/2}$, \eqref{eq:thirdTransSb'} becomes: 

\begin{equation}
    \Sb' = \beta^{-1}I + q(1-q)(\bm{\mu}'_+-\bm{\mu}'_-)(\bm{\mu}'_+-\bm{\mu}'_-)^\top, 
    \label{eq:thirdTransSb'asI+muDelta}
\end{equation}
where $\beta$ is an unknown scalar due to the proportionality $\hat{S}= \beta\Sh$. This shows that the transformation $\Vec{x}_i' \longleftarrow \hat{S}^{-1/2}\Vec{x}_i$ results in a covariance matrix that is equal to a rank-1 update on a scaled identity matrix. Note that if $\Sigma_+ \propto \Sigma_- \propto S$, the first transformation is proportional to $\Vec{x}_i' \longleftarrow S^{-1/2}\Vec{x}_i$, independent of $q$. The class imbalance $q$ must thus only be known if $\Sigma_+ \not\propto \Sigma_-$.

Finally, we need to show that the vector $\Vec{d}\propto\bm{\mu}'_+-\bm{\mu}'_-$ is indeed the eigenvector of $\Sb'$ with the largest eigenvalue. By construction, the rank-1 matrix $B = q(1-q)(\bm{\mu}'_+-\bm{\mu}'_-)(\bm{\mu}'_+-\bm{\mu}'_-)^\top$ has only one eigenvector $\Vec{d}\propto\bm{\mu}'_+-\bm{\mu}'_-$ with a non-zero eigenvalue $\lambda_{B,1} = q(1-q)(\bm{\mu}'_+-\bm{\mu}'_-)^\top(\bm{\mu}'_+-\bm{\mu}'_-)>0$. Note that $\Sb'$ is the matrix obtained by adding a scaled identity matrix $\beta^{-1} I$ to $B$. The impact of adding  $\beta^{-1} I$ is that all eigenvalues are incremented with $\beta^{-1}$, while the eigenvectors remain unchanged. Therefore, the eigenvector $\Vec{d}$ will still be the eigenvector of $\Sb'$ corresponding to the largest eigenvalue.
\end{proof}

\begin{algorithm}
\KwData{$N$ samples $\Vec{x}_i \in \R^{D\times1}$}
\caption{A summary of the MILDA algorithm}
\KwResult{Projection vector $\w \in \R^{D\times1}$}
$\Vec{x}_i' \longleftarrow f(\Vec{x}_i)$ as defined in Table \ref{tab:MILDAPreProc}\\
$\mb = \frac{1}{N}\sum_{i=1}^N \Vec{x}_i$,\\
$\Sb =  \frac{1}{N}\sum_{i=1}^N \Vec{x}_i\Vec{x}_i^\top - \mb\mb^\top$\\
$\w \propto \Sb^{-1}\mb$
\label{alg:MILDA}
\end{algorithm}

\begin{table}[h]
    \centering
    \begin{tabular}{l!{\color{black}\vrule}l}
         \textbf{Known}  & \textbf{Transformation} $f(\Vec{x}_i)$ \\
         \hline
         1) Class average $\mpos$ & $\Vec{x}_i' \longleftarrow \Vec{x}_i - \mpos$\\
         \arrayrulecolor{lightgray}\hline
         2) $\Vec{d} \propto \mpos-\mneg$ & $\Vec{x}'_i \longleftarrow \Vec{x}_i - (\mb - \alpha\Vec{d}), \alpha \neq 0$\\
         \arrayrulecolor{lightgray}\hline
         3) $S_+=\beta\Sigma_+,\ S_-=\beta\Sigma_-$ & $\Vec{x}_i' \longleftarrow (qS_+ + (1-q)S_-)^{-1/2}\Vec{x}_i$\\
         \phantom{3)} and $q$ &$\Vec{x}_i'' \longleftarrow \Vec{x}_i' - (\mb'-\alpha\Vec{d}), \alpha \neq 0$,\\
         \arrayrulecolor{lightgray}\hline
         4) $S \propto \Sigma_+ \propto \Sigma_-$ & $\Vec{x}_i' \longleftarrow S^{-1/2}\Vec{x}_i$\\
         &$\Vec{x}_i'' \longleftarrow \Vec{x}_i' - (\mb'-\alpha\Vec{d}), \alpha \neq 0$,\\
         &with $\Vec{d}$ defined in \eqref{eq:covCorrection}.
    \end{tabular}
    \caption{A summary of the transformations required given the known piece of information.}
    \label{tab:MILDAPreProc}
\end{table}

\section{Sensitivity to errors in the prior knowledge}
\label{sec:sensitivity}
Until now, we have always assumed that MILDA has access to the exact ground-truth statistic for each of the 4 cases in Table \ref{tab:MILDAPreProc}. In practice, often only an estimate of this statistic is available. In this section, we will therefore study how these approximation errors affect the projection vector from MILDA compared to the exact LDA projection vector \eqref{eq:objectiveMILDA}. More precisely, we will study these effects both in a mathematical and experimental context in Sections \ref{subsec:theoreticalError} and \ref{subsec:practicalError} respectively. Amongst others, we will show that such approximation errors matter less when classes are more difficult to separate or more imbalanced, which are precisely the cases that are the hardest for unsupervised models to tackle.

\subsection{Mathematical sensitivity analysis}
\label{subsec:theoreticalError}
Equation \eqref{eq:MILDA_pure} states that MILDA computes the projection vector:
\begin{equation}
    \Vec{w}_{MILDA} \propto \Sb^{-1}\mb, \label{eq:MILDA_pure_Rep}
\end{equation}
with $\mb,\ \Sb$ defined as in \eqref{eq:mb},\eqref{eq:SbOriginal}. Furthermore, from \eqref{eq:Sigma1Rank1}, we know that
\begin{align}
    \Sb &= \Sh + q(1-q)(\mpos - \mneg)(\mpos - \mneg)^\top \nonumber\\
        &\triangleq \Sh + q(1-q)\mD\mD^\top, \label{eq:mathErrorSb}
\end{align}
where we introduced a new notation $\mD \triangleq \mpos - \mneg$ for notational convenience.

Using \eqref{eq:mathErrorSb} in \eqref{eq:MILDA_pure_Rep}, applying the Sherman-Morrison identity \cite{Woodbury1950} and rearranging the order of the terms gives:
\begin{align}
    \Vec{w}_{MILDA} &\propto (\Sh^{-1} - \frac{q(1-q)\Sh^{-1}\mD\mD^\top\Sh^{-1}}{1+q(1-q)\mD^\top\Sh^{-1}\mD})\mb \nonumber\\
    & = \Sh^{-1}\mb - \frac{q(1-q)\mD^\top\Sh^{-1}\mb}{1+q(1-q)\mD^\top\Sh^{-1}\mD}\Sh^{-1}\mD \nonumber\\
    &= \Sh^{-1}(\mb + \phi\mD), \label{eq:wIFOmbandMd}\\
    &\text{with } \phi \triangleq - \frac{q(1-q)\mD^\top\Sh^{-1}\mb}{1+q(1-q)\mD^\top\Sh^{-1}\mD}. \nonumber
\end{align}
When comparing \eqref{eq:wIFOmbandMd} with \eqref{eq:LDA_MILDAObjective}, we find that MILDA optimises \eqref{eq:objectiveMILDA} if and only if $\mb$ is proportional to $\mD$. It can be verified that if $\mpos \sim \mneg$, then $\mb \sim \mD$. This indeed confirms Lemma \ref{lemma:MILDA_pure}, and this is what the different transformations in Section \ref{subsec:Transformations} aim to achieve. 

Let us now consider the case where the assumption $\mpos \sim \mneg$ (and thus $\mb \sim \mD$) is not (perfectly) satisfied. This happens for example if the transformations in Section \ref{subsec:Transformations} are computed based on approximations or estimates of the ground-truth statistics in the left column. In this case, we decompose $\mb$ into a component proportional to $\mD$ and a component orthogonal to $\mD$, i.e., $\mb = \delta\mD + \Vec{e}$, with $\delta \in \R$ a scaling term and $\Vec{e}$ an error vector orthogonal to $\mD$. $\Vec{w}_{MILDA}$ can then be written as:

\begin{equation}
    \Vec{w}_{MILDA} \propto \Sh^{-1}\left((\phi+\delta)\mD + \Vec{e}\right). \label{eq:wMildaNoAssumption}
\end{equation}

Depending on the statistics of the data, the same error does not always affect the projection vector $\Vec{w}$ equally. The error will have a small effect on $\Vec{w}_{MILDA}$ when the angle between $\Sh^{-1}\mD$ and $\Sh^{-1}\Vec{e}$ is small and/or when $\frac{1}{|\phi+\delta|}\frac{\|\Vec{e}\|_2}{\|\mD\|_2}$ is small. By using the definition of $\phi$ in \eqref{eq:wIFOmbandMd}, we can further expand the latter expression as: 

\begin{align}
    &\frac{1}{|\phi+\delta|}\frac{\|\Vec{e}\|_2}{\|\mD\|_2}\nonumber\\
    & = \frac{1}{\left|1-\frac{q(1-q)\mD^\top\Sh^{-1}(\mD+\frac{\Vec{e}}{\delta})}{1+q(1-q)\mD^\top\Sh^{-1}\mD}\right|} \frac{\|\Vec{e}\|_2}{\delta\|\mD\|_2} \nonumber\\
    & = \left|\frac{1+q(1-q)\mD^\top\Sh^{-1}\mD}{1-q(1-q)\mD^\top\Sh^{-1}\frac{\Vec{e}}{\delta}}\right| 
    \frac{\|\Vec{e}\|_2}{\|\mb-\Vec{e}\|_2}. \label{eq:sens3simplified}
\end{align}
Using \eqref{eq:wMildaNoAssumption} and \eqref{eq:sens3simplified}, we can conclude that the following properties reduce the sensitivity of MILDA to errors in the ground-truth statistics: 
\begin{itemize}
    \item Property 1: $\Sh^{-1}\Vec{e}$ and $\Sh^{-1}\mD$ form a small angle. Since $\Vec{e}$ and $\mD$ are orthogonal (by construction), more skewed class covariances can reduce the impact of the error (although this is not guaranteed). Note that the third and fourth transformations of Table \ref{tab:MILDAPreProc} first transform the data such that $\Sh=I$. For these transformations, $\Sh^{-1}\mD$ will thus always be orthogonal to $\Sh^{-1}\Vec{e}$. 
    \item Property 2: $\|\Vec{e}\|_2$ is small relative to $\|\mb\|_2$. This is trivial: the error matters less when it is smaller. The worst case is when $\mD \perp \mb$, since in this case $\mb=\Vec{e}$, such that $\w_{MILDA}\perp \w_{LDA}$ if $\Sh=I$. From \eqref{eq:mbIFOmpm}, we see that $\mD^\top\mb=q\|\mpos\|_2^2+(1-2q)\mpos^\top\mneg-(1-q)\|\mneg\|_2^2$. Since $\mD^\top\mb=0$ if $\mD \perp \mb$, this shows that this worst-case scenario does not occur when the class average of the class with the most samples is much larger than the other class average. If the classes are approximately balanced, it suffices that the norm of any one of the two class averages is much larger than the other. 
    \item Property 3: There is a large class imbalance. Indeed, if $q \approx 0$ or $q \approx 1$, the first term in \eqref{eq:sens3simplified} becomes minimal (i.e. close to 1), therefore reducing the sensitivity to errors. 
    \item Property 4: $\mD^\top\Sigma^{-1}\mD$ is small. This also minimises the first term in \eqref{eq:sens3simplified}. Interestingly, this implies that the error will have a reduced effect on $\w_{MILDA}$ when there is a lot of overlap between the two classes. MILDA has thus the nice property that approximation errors matter less when the classes are harder to separate, which are exactly the cases where fully unsupervised methods do not perform well. 
\end{itemize}
Beware that, in some cases, the beneficial effect of the latter two properties can be counteracted by other detrimental effects. This happens in the following cases.
\begin{itemize}
    \item A large class imbalance will increase the error after transformation 1 from Table \ref{tab:MILDAPreProc} (1 class average known) if the unknown class average is also the class with the least samples. Indeed, after the transformation the class with the most samples will be close to zero-mean resulting in a small $\mb'$, which directly counteracts the second aforementioned property. 
    \item Both a large class imbalance and a worse separability make it harder to correctly estimate $\Vec{d}$ when using transformations 3 or 4 from Table \ref{tab:MILDAPreProc} (the covariance methods), since both will reduce the impact of the rank-1 update on the overall covariance matrix in \eqref{eq:thirdTransSb'asI+muDelta}. 
\end{itemize}

\subsection{Experimental sensitivity analysis}
\label{subsec:practicalError}
\begin{figure*}[ht]
\begin{minipage}{\linewidth}
\begin{subfigure}[b]{\textwidth}
\centering
    \resizebox{0.6\columnwidth}{!}{
\begin{tikzpicture}
\begin{axis}[%
at={(0\textwidth,0\textwidth)},
clip mode=individual,
scale only axis,
 axis equal image,
tick align=inside,
x axis line style= {yshift=0.35cm, gray},
axis y line = none,
x tick style={yshift = 0.2cm, gray},
x tick label style={yshift= 0.35cm,anchor=south,font=\scriptsize\color{gray}},
axis background/.style={fill=none},
axis lines*=left,
axis on top = true,
xlabel={Angle $\alpha$ between $\w_{MILDA}$ and $\w_{LDA}$ [degrees]},
xlabel style={anchor=north,font=\small\color{gray}},
xmin=0,
xmax=90,
ymin=0,
ymax=10,
xtick={0,15,30,45,60,75,90},
]
\addplot graphics[xmin=0,ymin=-2,xmax=90,ymax=2] {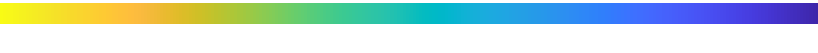};

\end{axis}
\end{tikzpicture}
}
\end{subfigure}
    \begin{subfigure}[b]{0.3\textwidth}
        \centering
        \input{Figures/SensitivityI.tex}
        \caption{$\Sh = I$}
        \label{fig:SenstivityI}
    \end{subfigure}
    \hfill
    \begin{subfigure}[b]{0.3\textwidth}
        \centering
        \input{Figures/Sensitivity5I.tex}
        \caption{$\Sh = 5I$}
        \label{fig:Senstivity5I}
    \end{subfigure}
    \hfill
    \begin{subfigure}[b]{0.3\textwidth}
        \centering
        \input{Figures/SensitivitySkew.tex}
        \caption{$\Sh = 0.3I + 0.7J$}
        \label{fig:SenstivitySkew}
        \end{subfigure}
    \caption{Heatmap representing the angle between $\w_{LDA}$ and $\w_{MILDA}$ as a function of the location of $\mpos$ and $\mneg$. The angle between the LDA and MILDA projection vector decreases when $\mneg<<\mpos$ or $\mneg>>\mpos$, when the covariance increases, and when it is less white.}
    \label{fig:Sensitivity}
\end{minipage}
\end{figure*}
To visualise the findings from Section \ref{subsec:theoreticalError}, the angle between the LDA and MILDA projection vectors is estimated in a two-dimensional space. To achieve this, we keep $\mpos$ fixed at $[0, 1]^\top$. Meanwhile, $\mneg = [x, y]^\top$ is varied in the ranges $x,y \in [-2,2]$. For each $\mneg$, we compute the projection vectors $\w_{MILDA} \propto \Sb^{-1}\mb$ and $\w_{LDA} \propto (\Sigma_+ + \Sigma_-)^{-1}(\mpos-\mneg)$, assuming balanced classes. We then measure the similarity between the two projections using the angle $\eta$ between $\w_{MILDA}$ and $\w_{LDA}$ using the formula $\cos(\eta) = \frac{\w_{LDA}^\top\w_{MILDA}}{\|\w_{LDA}\|_2\|\w_{MILDA}\|_2}$. The closer $\eta$ is to 0, the more the directions of $\w_{LDA}$ and $\w_{MILDA}$ are similar, and thus the more their respective projections are equivalent. 

This set-up allows us to simulate a wide array of possible violation of the assumption $\mpos \sim \mneg$ up to an irrelevant scaling and rotation\footnote{Scaling and rotation are irrelevant because LDA and MILDA transformations are affine invariant (proof omitted).}. To study how the class covariance influences the sensitivity of MILDA to estimation errors, we will repeat this process for three different sets of class covariances. In the first experiment, $\Sh_1 = I$. In the second experiment, $\Sh_2 = 5I$. Here, the variance is larger, making the classes less separable. In the third experiment, $\Sh_3 = 0.3I + 0.7J$, with $I$ the identity matrix and $J$ a matrix with all ones. In this case, the class covariances are very skewed, allowing us to verify that this indeed decreases the influence of the error in many cases. For each experiment, $\Sigma_+ = \Sigma_- = \Sh$ and $q=0.5$. To avoid any noise blurring the visualisation, no data was generated in this process. Instead, both the LDA and the MILDA projections were directly computed from $\Sigma_\pm$ and $\bm{\mu}_\pm$.

Figure \ref{fig:Sensitivity} confirms the insights obtained from the mathematical analysis in Section \ref{subsec:theoreticalError}. When $\Sh \propto I$ (Fig. \ref{fig:SenstivityI}), the blue circle at the positions where $\|\mpos\|\approx \|\mneg\|$ clearly shows the worst-case scenario where $\w_{MILDA} \perp \w_{LDA}$ in the area where $\mb = \Vec{e}$ (see property 2). The region where $\mpos\approx-\mneg$ is particularly interesting, as here the alignment assumption $\mpos\sim\mneg$ is quite well satisfied, yet the sensitivity of MILDA to small deviations from such perfect alignment is very high because $\mb \approx \Vec{0}$.

The ideal location for $\mneg$ to be, is either around $[0,0]^\top$ (if $\mpos$ is non-zero) or $\pm \kappa\mpos$ with $\kappa \gg 1$ In these regions, the misalignment error between $\mpos$ and $\mneg$ can be larger without having much effect on the solution.  

Finally, we also clearly observe that MILDA becomes more robust against errors as the class covariances increase (Figure \ref{fig:Senstivity5I}) or become less white (Figure \ref{fig:SenstivitySkew}), as predicted in Property 1 and Property 4 in the mathematical error analysis. This is interesting, as these are the typical scenario's where unsupervised methods suffer the most.

\section{Experiments}
\label{sec:experiments}
In this section, we demonstrate that MILDA closely matches the performance of a fully supervised LDA. In the first set of experiments, the two models are compared to each other on a vast set of simulated data with varying parameters. This allows us to accurately investigate the influence of each parameter on the model's performance. In a second experiment, we apply MILDA in the context of target detection using a sensor array. This is a common problem for e.g. radar detection, wireless communication and brain-computer interfaces \cite{Markovich2009, Gu2012, Khabbazibasmenj2012}. Finally, MILDA and LDA are compared in a third experiment on non-stationary data, highlighting how the label-free MILDA model can easily adapt to changing statistics. 

Note that the LDA and MILDA models would not be used interchangeably. LDA requires all statistics to be known, whereas MILDA is specifically designed for cases where much less information is available. Therefore, LDA should be seen as a benchmark that MILDA should ideally approach as closely as possible (except in the third experiment).

\subsection{Ablation study}
\label{subsec:AblationStudy}
\begin{table*}[t]
    \centering
    \begin{tabular}{l!{\color{black}\vrule}l!{\color{black}\vrule}l}
         \textbf{Known}                     &  \textbf{K-means update}                  & \textbf{GMM update} \\
         \hline
         1) Class average $\mpos$           &   $\bm{\mu}_1^{(i)} \longleftarrow (1-\lambda)\bm{\mu}_1^{(i)} + \lambda\mpos$                                        & $\bm{\mu}_1^{(i)} \longleftarrow (1-\lambda)\bm{\mu}_1^{(i)}+\lambda\mpos$\\
         \arrayrulecolor{lightgray}\hline
         2) $\Vec{d} \propto \mpos-\mneg$   &   $\bm{\mu}_{1/2}^{(i)} \longleftarrow \left((1-\lambda)I + \lambda\Vec{d}\Vec{d}^\top\right)(\bm{\mu}_{1/2}^{(i)}-\mb)$    & $\bm{\mu}_{1/2}^{(i)} \longleftarrow \left((1-\lambda)I + \lambda\Vec{d}\Vec{d}^\top\right)(\bm{\mu}_{1/2}^{(i)}-\mb)$\\
         \arrayrulecolor{lightgray}\hline
         3) $\Sigma_+,\ \Sigma_-$ and $q$ or  &   $\bm{\mu}_{1/2}^{(i)} \longleftarrow \bm{\mu}_{1/2}^{(i)}$                                                          & $\Sigma_{1/2}^{(i)} =  (1-\lambda)\Sigma_{1/2}^{(i)}+\lambda\Sigma_{+/-}$\\
         4) $\Sigma_+ \propto \Sigma_-$ &&
    \end{tabular}
    \caption{The updating schemes for K-means and GMM to incorporate known pieces of information as a prior, with $\lambda\in [0,1]$ and $\|\Vec{d}\|_2=1$. By default, $\lambda=0.7$ is chosen, as it empirically produced the best results. In contrast to MILDA, $\Sigma_\pm$ were used for the GMM instead of $S_\pm\propto\Sigma_\pm$ in transformations 3 and 4, since the algorithm is not invariant to scaling. Classes 1/2 are paired with classes +/- after the K-means++ initialisation such that $\|\mpos-\bm{\mu}^{(1)}_1\|_2$ is minimal when $\mpos$ is known and such that $\|\Sigma_+-\Sigma_1^{(1)}\|_F+\|\Sigma_- -\Sigma_2^{(1)}\|_F$ is minimal when $\Sigma_+,\ \Sigma_-$ are known. No explicit pairing is necessary when $\Vec{d}$ is known.}
    \label{tab:GMMKmeansUpdate}
\end{table*}

\subsubsection{Experiment design}
In this experiment, the models are validated on a simulated dataset where the two classes are sampled from two different Gaussian distributions. By default, $N=1000$ samples are sampled from the distributions using the following hyperparameters: $\mpos = [1, \Vec{0}]^\top \in \R^{D\times1}$, $\mneg = [-1,\Vec{0}]^\top \in \R^{D\times1}$, $\Sigma_+ = (1-\rho)I +\rho J$, with $I \in \R^{D\times D}$ the identity matrix and $J \in \R^{D\times D}$ a matrix of all ones. $\rho \in [0,1]$ controls how correlated the different features are ($\rho=0.3$ by default). $D=10$ is the data dimensionality. $\Sigma_-$ is computed by using the same eigenvalues and eigenvectors as $\Sigma_+$, but pairing each eigenvector with a different eigenvalue than in $\Sigma_+$. This generates two classes with the same covariance structure but rotated such that the maximal covariance directions are orthogonal. By default, both classes are balanced ($q=0.5$).

Each time, one specific parameter in the model is varied to study its influence on the behaviour of MILDA. The following parameters are studied: 
\begin{itemize}
    \item The effect of the class separation: $\mneg = [i, \Vec{0}]^\top, i\in [-3,1]$ (results shown in Fig. \ref{fig:ClassSeparation}).
    \item The effect of the correlation between the features $\rho \in [0,1]$ (results shown in Fig. \ref{fig:Correlation}).
    \item The degree to which $\Sigma_+$ and $\Sigma_-$ are similar (results shown in Fig. \ref{fig:Orthogonality}). This is controlled by choosing $\Sigma_- = (1-\sigma) \Sigma_+ + \sigma \Sigma_{-,default}, \sigma \in [0,1]$.
    \item The data dimensionality $D \in [2,100]$ (results shown in Fig. \ref{fig:Dimensions}).
    \item The influence of the class imbalance, modelled by the fraction of samples from the positive class $\C_+$ $q \in [0.01,0.99]$ (results shown in Fig. \ref{fig:ClassImbalance}).
\end{itemize}

Given that MILDA is a label-free method, we also compare it to two other popular unsupervised classification methods: K-means and Gaussian mixture models (GMMs). Where possible, these models were modified\footnote{For all experiments in this paper, we have verified that these modifications do not deteriorate the results compared to the original K-means and GMM algorithm. Using all default parameters, the accuracy of the original and adapted K-means algorithms are respectively 64\% and 84\%. The accuracy of the original and adapted GMM algorithms are respectively 66\% and 69\%.} to incorporate the same prior knowledge as MILDA to make the comparison as fair as possible. This is achieved by updating specific parameters at each iteration, following the updating procedure from Table \ref{tab:GMMKmeansUpdate}. All other design choices for K-means and GMM are identical to the default Matlab R2021a implementation. While LDA is the only true benchmark here (as this is MILDA's target performance), these additional comparisons with K-means and GMMs are added to provide some insights into the strengths and limitations of MILDA as a label-free classifier compared to traditional unsupervised methods. 

The LDA model is always trained on the ground-truth statistics that are used to generate the samples, making it the \textit{ideal} LDA model. Similarly, MILDA always uses the known ground-truth statistics to compute the transformations from Table \ref{tab:MILDAPreProc}. This reflects real-life limitations, as sample statistics of test data are generally unknown.

Each experiment is repeated 50 times for the three ground-truth statistics (one class average known, the relative difference between class averages known and the class covariances known). The results are then averaged and shown in Figures \ref{fig:ClassSeparation} to \ref{fig:ClassImbalance}. The shaded areas represent the standard deviation for each model.

\subsubsection{Class separation}
Unsurprisingly, Figure \ref{fig:ClassSeparation} shows that a better class separation improves the performance of all models. MILDA performs significantly better than K-means and GMM when the class separation is large enough (in this example when $\|\mpos - \mneg\|_2>0.1$). MILDA is slightly worse than LDA for most accuracies. The difference in their accuracy is at most 1\%, except for the case where only the class covariances are known (cases 3/4) if the class separation is small. In that case, the MILDA model is very sensitive to disparities between the ground-truth covariances and the estimated sample covariances. However, K-means and GMM models never outperform MILDA, as they suffer even more from such data with poorly separated classes. 

\subsubsection{Correlation between features}
From Figure \ref{fig:Correlation}, we again see that MILDA and LDA achieve very similar performances, no matter how correlated the features are. K-means is not significantly worse than MILDA when both class covariances are perfectly white, but quickly drops in performance when this is no longer the case. It is significantly worse than MILDA when the correlation between features is high (in this example when $\rho>0.1$). The GMM improves as features become more correlated, but never comes near the performance of LDA or MILDA. 

\subsubsection{Similarity between covariance matrices}
The degree to which $\Sigma_-$ and $\Sigma_+$ are different has no influence on either MILDA or LDA, as shown in Figure \ref{fig:Orthogonality}. MILDA thus reaches similar performances as LDA, no matter how different the class covariances are. K-means does improve as the covariances become more orthogonal. 

\subsubsection{Data dimensionality}
While the performance of K-means and GMM quickly drops with an increasing data dimensionality in Figure \ref{fig:Dimensions} due to the curse of dimensionality \cite{Bouveyron2019}, MILDA only slowly loses performance. The slight drop in performance for MILDA compared to LDA is caused by the cumulative sum of small disparities between the ground-truth and sample statistics over an increasing data dimensionality. Note that LDA has access to all ground-truth statistics and thus doesn't suffer from these estimation errors. However, in reality, its statistics are often estimated on labelled data. In this case, LDA will also suffer from the curse of dimensionality. 

\subsubsection{Class imbalance}
Although LDA and MILDA no longer optimise the same objective function when $q\neq 0.5$ and $\Sigma_+ \not\propto \Sigma_-$, as explained in Section \ref{sec:MILDA}, they still produce very similar results. This is shown in Figure \ref{fig:ClassImbalance}. MILDA is the most different from LDA in the cases where only 1 class average is known (case 1) and when there are very few samples from the other class. This is expected, as here, MILDA knows hardly anything about the second class, making it very difficult to properly estimate the projection vector, as explained in Section \ref{sec:sensitivity}. A similar effect can be seen in K-means. All models reach near 100\% accuracy when the classes are extremely imbalanced, which is simply because the chance level also reaches 100\%.  

\begin{figure*}[p]
\begin{minipage}{\linewidth}
    \begin{subfigure}[b]{0.3\textwidth}
        \centering
        \input{Figures/3Models_mu_class_separation.tex}
        \caption{$\mpos$ known}
        \label{fig:class_separation_mu}
    \end{subfigure}
    \hfill
    \begin{subfigure}[b]{0.3\textwidth}
        \centering
        \input{Figures/3Models_Delmu_class_separation.tex}
        \caption{$\Vec{d}$ known}
        \label{fig:class_separation_Delmu}
    \end{subfigure}
    \hfill
    \begin{subfigure}[b]{0.3\textwidth}
        \centering
        \input{Figures/3Models_cov_class_separation.tex}
        \caption{$\Sigma_+,\ \Sigma_-,\ q$ known}
        \label{fig:class_separation_cov}
    \end{subfigure}
    \caption{Accuracy as a function of class separation. All models improve as the class separation grows. LDA and MILDA perform almost equally well, except when only the covariance is known and the class separation is small (Figure \ref{fig:class_separation_cov}). K-means and GMM perform significantly worse for class separations above 0.1.}
    \label{fig:ClassSeparation}
\end{minipage}
\begin{minipage}{\linewidth}
    \begin{subfigure}[b]{0.3\textwidth}
        \centering
        \input{Figures/3Models_mu_correlation.tex}
        \caption{$\mpos$ known}
    \end{subfigure}
    \hfill
    \begin{subfigure}[b]{0.3\textwidth}
        \centering
        \input{Figures/3Models_Delmu_correlation.tex}
        \caption{$\Vec{d}$ known}
    \end{subfigure}
    \hfill
    \begin{subfigure}[b]{0.3\textwidth}
        \centering
        \input{Figures/3Models_cov_correlation.tex}
        \caption{$\Sigma_+,\ \Sigma_-,\ q$ known}
        \label{subfig:corrCov}
    \end{subfigure}
    \caption{Accuracy as a function of the correlation between features. If features are more correlated, the classes tend to become more separable. LDA, MILDA and GMM, therefore, improve in performance. K-means drops in performance because it implicitly assumes white class covariance matrices. LDA and MILDA perform nearly equally well for all settings, outperforming K-means and GMM.}
    \label{fig:Correlation}
\end{minipage}
\begin{minipage}{\linewidth}
    \begin{subfigure}[b]{0.3\textwidth}
        \centering
        \input{Figures/3Models_mu_orthogonality.tex}
        \caption{$\mpos$ known}
    \end{subfigure}
    \hfill
    \begin{subfigure}[b]{0.3\textwidth}
        \centering
        \input{Figures/3Models_Delmu_orthogonality.tex}
        \caption{$\Vec{d}$ known}
    \end{subfigure}
    \hfill
    \begin{subfigure}[b]{0.3\textwidth}
        \centering
        \input{Figures/3Models_cov_orthogonality.tex}
        \caption{$\Sigma_+,\ \Sigma_-,\ q$ known}
    \end{subfigure}
    \caption{Accuracy as a function of how similar the class covariances are. Whether or not the two classes have similar class covariances or not has little influence on the performance of any model. Only K-means improves noticeably as the two class covariances become more different.}
    \label{fig:Orthogonality}
\end{minipage}
\end{figure*}
\begin{figure*}[t]
\begin{minipage}{\linewidth}
    \begin{subfigure}[b]{0.3\textwidth}
        \centering
        \input{Figures/3Models_mu_dimensions.tex}
        \caption{$\mpos$ known}
    \end{subfigure}
    \hfill
    \begin{subfigure}[b]{0.3\textwidth}
        \centering
        \input{Figures/3Models_Delmu_dimensions.tex}
        \caption{$\Vec{d}$ known}
    \end{subfigure}
    \hfill
    \begin{subfigure}[b]{0.3\textwidth}
        \centering
        \input{Figures/3Models_cov_dimensions.tex}
        \caption{$\Sigma_+,\ \Sigma_-,\ q$ known}
    \end{subfigure}
    \caption{Accuracy as a function of the data dimensionality. GMM and K-means models suffer more from the curse of dimensionality than MILDA. LDA, having access to all ground-truth statistics, is unaffected.}
    \label{fig:Dimensions}
\end{minipage}
\begin{minipage}{\linewidth}
    \begin{subfigure}[b]{0.3\textwidth}
        \centering
        \input{Figures/3Models_mu_class_imbalance.tex}
        \caption{$\mpos$ known}
        \label{fig:class_imbalance_mu}
    \end{subfigure}
    \hfill
    \begin{subfigure}[b]{0.3\textwidth}
        \centering
        \input{Figures/3Models_Delmu_class_imbalance.tex}
        \caption{$\Vec{d}$ known}
        \label{fig:class_imbalance_Delmu}
    \end{subfigure}
    \hfill
    \begin{subfigure}[b]{0.3\textwidth}
        \centering
        \input{Figures/3Models_cov_class_imbalance.tex}
        \caption{$\Sigma_+,\ \Sigma_-,\ q$ known}
    \end{subfigure}
    \caption{Accuracy as a function of the class imbalance. Although MILDA and LDA optimise a slightly different objective when the classes are not balanced, their difference in performance is nearly nonexistent, no matter the class imbalance.}
    \label{fig:ClassImbalance}
\end{minipage}
\end{figure*}
\subsection{Example use case: target signal detection}
\label{Subsec:targetdetection}
\subsubsection{Experiment design} To illustrate the use of MILDA in a concrete application use case, we consider a target signal detection problem in a D-channel sensor array. Given a target signal $s(t)$, the observed signal $\Vec{x}(t) \in \R^{D\times1}$ can be modeled as $\Vec{x}(t) = s(t)\Vec{a}+\Vec{n}(t)$, with $\Vec{a}\in\R^{D\times1}$ the steering vector and $\Vec{n}(t)\in\R^{D\times1}$ the noise. All signals are assumed to be stationary. To illustrate the different ways that MILDA can be used, we propose three sub-problems: 

\begin{itemize}
    \item In the first problem, we attempt to detect whether or not the target signal is present, i.e. $s(t) \in \{0,1\}$. The steering vector $\Vec{a}$ is unknown, but assumed fixed within the observation window that contains the $N$ samples that have to be classified. Even though the general noise statistics are unknown, the sensors are calibrated such that the noise is zero-mean. One class average ($\bm{\mu}_{noise}=\Vec{0}$) is thus known, allowing us to use MILDA with transformation 1 from Table \ref{tab:MILDAPreProc}. 
    \item In the second problem, a binary signal $s(t) \in \{-1/2,1/2\}$ is sent from a single location. In this case, the signal $s(t)$ is always active (i.e., it is never zero). The steering vector $\Vec{a}$ is known (up to an arbitrary scaling) \cite{Markovich2009}. The statistics of $s(t)$ and $\Vec{n}(t)$ are assumed to be unknown. In this case, we know that the relative difference between both classes averages is proportional to the steering vector ($\mpos-\mneg\propto \Vec{a}$). We can therefore use MILDA with transformation 2 from Table \ref{tab:MILDAPreProc}.
    \item In the third problem, two different locations interchangeably transmit the signal $s(t)=1$ while the other location is silent. For each observation $\Vec{x}(t)$, we aim to detect which location transmitted the signal. The locations (and corresponding steering vectors) are unknown, similar to case 1. The noise covariance matrix $S=\E [\Vec{n}(t)\Vec{n}^\top(t)]$ is assumed to be known.  It is therefore possible to use MILDA with transformation 4 from Table \ref{tab:MILDAPreProc}.
\end{itemize}
We set $D=10$ and perform multiple Monte-Carlo runs. In each run, $N=1000$ samples are created with equal probability between the two detection classes of the corresponding experiment, i.e., $N/2$ of the samples belong to class 0, the other $N/2$ belong to class 1. The noise is modelled as $\Vec{n}(t) = \Vec{g}(t) + \Vec{f}(t)$, with Gaussian noise $\Vec{g}(t) \sim \mathcal{N}(0.3 \diag(\bm{\sigma}) + 0.7J)$ and a high-frequency interference signal $\Vec{f}(t) = \bm{\alpha}\odot\sin(\pi/2t + \bm{\phi})$. The variances $\bm{\sigma}\in\R^{D\times1}$ and amplitudes $\bm{\alpha}\in\R^{D\times1}$ are randomly sampled from the uniform distribution $\mathcal{U}_D(\Vec{0},\Vec{1})$. The phase for each dimension $\bm{\phi}\in\R^{D\times1}$ is sampled from $\mathcal{U}_D(\Vec{0},\Vec{2}\bm{\pi})$. The steering vectors are sampled from $\mathcal{U}_D(-\Vec{1},\Vec{1})$. Each experiment is repeated 100 times. All models are trained in the same way as in Section \ref{subsec:AblationStudy}. Since LDA has access to information that wouldn't be available in the described scenarios, it merely serves as a benchmark. 

This experiment is conducted with the primary objective of demonstrating the practical applicability of MILDA within a specific use-case context. The intent is not to assert superiority over state-of-the-art target detection algorithms. The experiments are purposefully simplified for clarity.  

\subsubsection{Results}
The detection accuracy for each model and each experiment is displayed in Table \ref{tab:resultSteering}. Similar to the previous experiments, LDA and MILDA achieve comparable performances and consistently outperform the GMM and K-means models. The inclusion of non-Gaussian noise also demonstrates that MILDA does indeed not require the assumption of Gaussianity to be comparable to LDA. 

\begin{table}[ht]
    \centering
    \begin{tabular}{c|c c c c}
        Problem & LDA  & MILDA  & K-means & GMM \\
        \hline
        Zero-mean noise     & 93.4 (4.4) & 92.8 (4.4)& 63.4 (5.6) & 87.5 (14.3)\\
        Binary signal       & 83.2 (4.4) & 82.9 (4.4)& 76.6 (5.9) & 64.4 (9.1)\\
        Noise covariance    & 90.3 (4.6) & 90.2 (4.6)& 64.6 (11.1)& 63.9 (10.2)\
    \end{tabular}
    \caption{The average accuracy (\%) and standard deviation between brackets for each model. MILDA performs only slightly worse than LDA, even though it requires significantly less information. K-means and GMM always perform significantly worse, even with access to the same ground-truth information as MILDA.}
    \label{tab:resultSteering}
\end{table}

\subsection{Comparison on non-stationary data}
\subsubsection{Experiment design}
Until now, all experiments were performed on stationary data. However, there is a vast amount of classification problems where the stationarity assumption does not hold. Instead, the class distributions can shift and morph over time. Ideally, the classification models should adapt to these changes. While this is generally possible for MILDA (assuming the ground-truth statistic is stationary), this is rarely the case for LDA. It rarely has access to labels to re-estimate the relevant statistics at inference time to recompute the LDA projection vector.

To build further on the first problem from Section \ref{Subsec:targetdetection}, we illustrate the effects of non-stationary data on binary target signal detection in zero-mean noise by introducing three changes in the data and tracking the accuracy over time:
\begin{itemize}
    \item The first change is a sudden change in the noise covariance from $\Sigma_1$ to $\Sigma_2$. 
    \item The second change is a \textit{gradual} drift in both the steering vector from $\Vec{a}_1$ to $\Vec{a}_2$ and the noise covariance from $\Sigma_2$ to $\Sigma_3$. 
    \item The third change is a sudden change in the steering vector from $\Vec{a}_2$ to $\Vec{a}_3$.
\end{itemize}

Similar to the first problem, $s(t) \in \{0,1\}$ and $\Vec{a} \sim \mathcal{U}_D(-\Vec{1},\Vec{1})$. The noise is modelled in the same way as in Section \ref{Subsec:targetdetection}. The signal and noise statistics are changed by respectively resampling $\Vec{a}$ and $\bm{\sigma,\ \alpha,\ \phi}$. During the gradual change, each parameter is changed by linearly interpolating between the value at the start and at the end of the gradual change. The norm of the steering vector is also changed by linearly interpolating between $\|\Vec{a}_1\|$ and $\|\Vec{a}_2\|$. This guarantees that the expected class separation is constant during the gradual shift.

We performed 1000 Monte-Carlo runs. $10 000$ samples are classified during each run, with a sudden change in noise covariance at sample $2500$, a gradual change between samples $5000$ and $7500$ and a sudden change in the steering vector at sample $7500$. MILDA, K-means and GMM are adaptively retrained using a moving window covering the last $500$ samples. LDA is trained using the original statistics $\Sigma_1$ and $\Vec{a}_1$ and thus does not adapt to the changes in statistics (as there are no labels available to update the model). Since the noise is always zero-mean, MILDA, K-means and GMM use $\bm{\mu}_{noise}=\Vec{0}$ as the stationary ground-truth statistic. 

\subsubsection{Results}
The average accuracy of the four models is shown in Figure \ref{fig:steeringVectorAdaptive}. As expected, LDA slightly outperforms MILDA on samples that are drawn from the same distribution as it was trained on. LDA and MILDA equally drop in performance when sudden changes are introduced in the statistics. However, MILDA completely recovers after $500$ samples, which is the width of the moving window MILDA was trained on. Remarkably, MILDA barely drops in performance during the gradual change in statistics (on average only $0.25\%$), showing that the model is well able to track gradual changes in statistics. This is also true for GMM and K-means, although they always perform significantly worse than MILDA ($p<1e-5$). As expected, LDA continuously drops in performance as more statistics are changed. Since it is unable to adapt due to the lack of labels, it is also unable to recover from these changes. 

\begin{figure}
    \centering
    \input{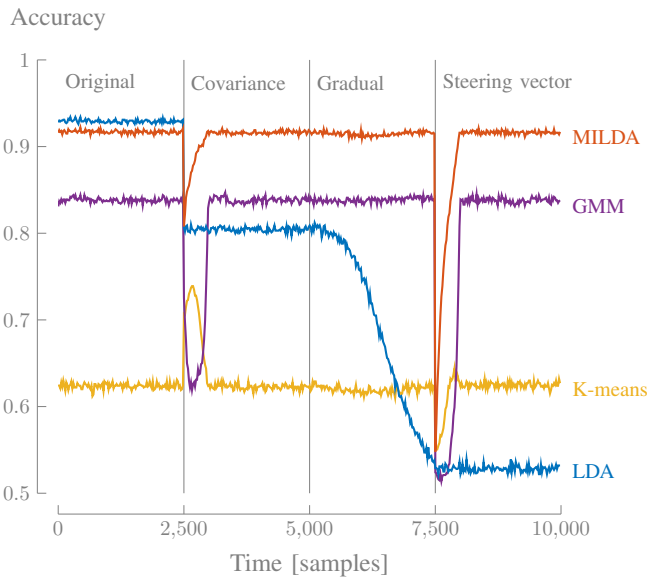}
    \caption{Comparison of LDA, MILDA, K-means and GMM on non-stationary data. Since MILDA, K-means and GMM do not require labels to train, it can adapt to both sudden and gradual changes in the statistics.}
    \label{fig:steeringVectorAdaptive}
\end{figure}
\section{Computational cost}
\label{sec:CompCost}
Simple models such as K-means and LDA are often chosen over more complex models such as neural networks because of their low computational cost, which is particularly relevant in adaptive realizations where the classifier is continuously updated when new data becomes available. It is therefore relevant to study whether MILDA has a similar computational cost. Table \ref{tab:CompCost} shows the average time (in ms) each model required to train using all default parameters from Section \ref{sec:experiments}. MILDA is only slightly slower than LDA, but this is mainly because LDA has access to the ground-truth covariance matrices. If LDA also needs to estimate $\Sigma_+$ and $\Sigma_-$ from data, there is no longer a difference in computation time. The K-means and GMM models take over 10x longer to compute in this specific experiment. 

To examine the computational cost to train each model at much larger scales, it can be useful to examine their asymptotic complexity. It is well known that LDA has a time complexity of $\mathcal{O}(ND^2) + \mathcal{O}(D^3)$, with $N$ the number of samples and $D$ the number of features (i.e. the dimension of $\Vec{x}_i$). Similarly, K-means has a time complexity of $\mathcal{O}(ND)$ \cite{Arthur2007} and GMM's have a time complexity of $\mathcal{O}(ND^2I) + \mathcal{O}(D^3I)$, with $I$ the number of iterations that GMM's train (assuming naive computations of $\Sigma^{-1}$ for LDA and GMM's). MILDA has the same time complexity as LDA, namely $\mathcal{O}(ND^2) + \mathcal{O}(D^3)$. Since GMM's need to make similar computations as MILDA, but repeated at every iteration, a GMM model will always be significantly slower than a MILDA model. MILDA is also expected to be faster than K-means if the data dimensionality $D$ is relatively small. In the aforementioned experiments, K-means was faster than MILDA for a data dimensionality $D>80$ at $N=1000$ samples and $D>120$ for $N=10000$ samples (although K-means no longer performed better than chance in these cases, while MILDA barely dropped in performance, as was shown in Figure \ref{fig:Dimensions}).

When MILDA is used in an adaptive context, its computational cost can be further reduced to $\mathcal{O}(D^2)$ per iteration by using similar techniques as used in, e.g., recursive least-squares (RLS) \cite{Hastie2009}. Furthermore, only the global covariance matrix and global average need to be stored in memory, making the model much less memory-intensive than e.g. an adaptive K-means model, where all the data needs to be stored.

Even though the transformations required in MILDA do require some additional computations compared to a simple LDA model, these transformations are much cheaper to compute than the estimation of the covariance matrices (assuming $N\gg D$). MILDA is, therefore, not significantly more expensive than LDA when LDA also needs to estimate the covariance matrix from the data. 

\begin{table}[ht]
    \centering
    \begin{tabular}{c|c c c c}
        Prior knowledge & LDA  & MILDA  & K-means & GMM \\
        \hline
        $\mpos$         & 0.24 & 0.31 & 3.0 & 4.7\\
        $\Vec{d}$       & 0.24 & 0.44 & 5.6 & 4.7\\
        $\Sigma_{+/-}$  & 0.24 & 0.46 & 5.7 & 4.8\
    \end{tabular}
    \caption{MILDA and LDA require about 10x less time to classify the data than K-means and GMM (time in [ms]).}
    \label{tab:CompCost}
\end{table}
\section{Conclusion \& Future Work}
\label{sec:conclusion}
When discriminating between two classes, (Fisher's) Linear Discriminant Analysis (LDA) requires knowledge of the class averages and class covariance matrices to work optimally, which usually have to be estimated from the data based on available class labels. If no labels are available, one would need to know all ground-truth statistics up to a high precision to use LDA, which rarely is the case in practice. However, at least one ground-truth statistic is often known, e.g. a single class average, the relative difference between two class averages or the class covariance matrices. In this paper, we have proposed the minimally informed LDA (MILDA) model that can calculate the same projection vector as LDA without requiring labelled data, assuming prior knowledge of only one of these ground-truth statistics. 

We have mathematically proven that the MILDA projection vector is equivalent to the LDA projection vector under the mild assumption that the classes are balanced or that the class covariances have a similar shape. When both assumptions are violated, MILDA optimises a slightly different objective function, yet we could not find a noteworthy difference in performance between both. Furthermore, we have shown both mathematically and empirically that, in general, MILDA has the nice property of being quite robust to approximation errors in cases where the classes are harder to separate, the class covariance matrices are less white or when the classes are more unbalanced, which are all cases where unsupervised methods typically have difficulties with. 

The performance of MILDA is comparable to the performance of an ideal LDA projection vector in nearly all studied cases. The relative difference in performance is nearly unaffected by the class separation, the correlation between different features, how similar the covariance matrices of both classes are or the class imbalance. MILDA is somewhat sensitive to outliers, except for the case where the relative difference between the class averages is known.

The MILDA projection vector can be computed in closed form with a computational cost comparable to the LDA projection vector. Because MILDA has a low computational cost to train and does not require any labels, it is well-suited for use in adaptive classifiers that update themselves over time using the most recent data.  

Although MILDA is promising on its own and can be used in a wide array of applications, the necessity to know at least one ground-truth statistic reduces the utility of MILDA in some research fields. To go from a minimally informed model to a truly unsupervised model, MILDA should replace the necessity to know one ground-truth statistic with an assumption on the data. Some assumptions are trivial to incorporate in MILDA, such as the assumption that the class covariances are white or that one class average is zero (respectively use case 4 and case 1 from Table \ref{tab:MILDAPreProc}). Other assumptions, such as the assumption that the classes are Gaussian or that the features are uncorrelated, but not white, are not as trivial to incorporate. Therefore, further research on how these assumptions can be leveraged in MILDA can be vital to further expand the applicability of MILDA. 
\bibliographystyle{IEEEtran}
\bibliography{references}

\end{document}